\documentclass[11pt]{amsart}

\usepackage[a4paper,margin=0.92in]{geometry}
\usepackage[T1]{fontenc}
\usepackage[utf8]{inputenc}
\usepackage{lmodern}
\usepackage{microtype}
\usepackage{amsmath,amssymb,amsfonts,amsthm,mathtools,bm}
\usepackage{booktabs,tabularx,array,multirow}
\usepackage{graphicx}
\usepackage{enumitem}
\usepackage[hidelinks]{hyperref}
\usepackage[nameinlink,noabbrev]{cleveref}

\allowdisplaybreaks
\setlist[itemize]{leftmargin=*,itemsep=1pt,topsep=2pt}
\setlist[enumerate]{leftmargin=*,itemsep=1pt,topsep=2pt}
\newcolumntype{Y}{>{\raggedright\arraybackslash}X}

\newtheorem{theorem}{Theorem}[section]
\newtheorem{proposition}[theorem]{Proposition}
\newtheorem{corollary}[theorem]{Corollary}
\newtheorem{lemma}[theorem]{Lemma}
\theoremstyle{definition}

\theoremstyle{remark}

\newcommand{\R}{\mathbb R}
\newcommand{\U}{\mathcal U}
\newcommand{\Pcal}{\mathcal P}
\newcommand{\B}{\mathbb B}
\newcommand{\rank}{\operatorname{rank}}
\newcommand{\row}{\operatorname{row}}

\newcommand{\eps}{\varepsilon}
\newcommand{\ceil}[1]{\left\lceil #1\right\rceil}
\newcommand{\norm}[1]{\left\lVert #1\right\rVert}
\newcommand{\rnum}{r}

\title[Task-information--state frontiers]
{What Must Survive?\\
Exact Task-Information--State Frontiers Under Partial Task Revelation}

\author{Ronald Katende}
\address{Department of Mathematics, Kabale University, Kabale, Uganda}
\email{rkatende@kab.ac.ug}
\date{}


\begin{document}

\begin{abstract}
A state may need to be compressed before its exact downstream task is known. For a finite family of linear task operators $\{T_u\}_{u\in\U}$, we characterize exactly how much state must survive when only one of $K$ coarse task messages is available at compression time:
\[
p^*(K)=
\min_{\substack{\Pcal\text{ partition of }\U\\|\Pcal|\le K}}
\max_{C\in\Pcal}\rank(T_C),
\]
where $T_C$ stacks the tasks that remain unresolved within cell $C$. Thus task information reduces retained state precisely by separating tasks whose joint task-visible subspace is expensive. We derive the dual bit frontier, an irreducible common-core floor, and a sharp singular-value characterization for nonzero tolerance. For overlapping task families, a cumulative overlap deficit places the exact frontier within an additive $\Delta$ of weighted load balancing. Optimal partitioning is strongly NP-hard, yet direct-sum families admit classical approximation guarantees and a PTAS.

We stress-test the theory on learned task operators from CIFAR-10, Fashion-MNIST, SVHN, and Burgers observables. For Burgers with $K=2$ and $\eps=0.35$, a constructive partition attains $p=4$, while balanced random partitions require median $p=6$, with held-out worst standardized RMSE $0.0368$. The results distinguish uniform task sufficiency from distributional and decision-level performance.
\end{abstract}

\maketitle

\section{Problem and information order}

Compression is usually priced after the downstream task or workload has already been fixed. Predictive-state representations, functional compression, task-aware inference, workload sketches, and task-conditioned memory all exploit some form of downstream relevance \cite{LittmanSuttonSingh2001,FeiziMedard2009,Nakanoya2021,Woodruff2014,Zhou2025,Huang2026}. A different regime appears when the state must be formed first, some coarse information about its future use is available, and the exact use is disclosed only later. The missing object is then not a single-task sufficient statistic and not one sketch for a known workload. It is the frontier between \emph{how much of the future task is known now} and \emph{how much state must survive until the task is known exactly}.

We isolate that question in the smallest model in which it can be answered sharply. Let $\U=\{1,\ldots,m\}$ and let
\[
T_u:\R^n\to\R^{r_u},\qquad u\in\U,
\]
be linear task operators. The source state satisfies $x\in\B_2^n$. Before compression, an oracle reveals a deterministic advice symbol $a(u)\in[K]$, with $K\le2^b$. The encoder knows $a(u)$ but not $u$. After compression, the exact task is revealed and a task-specific decoder must return $T_ux$. For an advice cell $C=a^{-1}(j)$, the encoder is a continuous map $\phi_C:\B_2^n\to\R^p$ and the decoder for task $u\in C$ is $\psi_u:\R^p\to\R^{r_u}$. The retained-state budget is the common coordinate dimension $p$.

The information order is the point. If $u$ were known before compression, the encoder could retain only task $u$. If $u$ were never revealed, all tasks in a cell would have to be decoded jointly. Here the encoder must preserve every distinction that can still matter inside the unresolved advice cell, but no distinction that is irrelevant to all tasks in that cell.

For $C\subseteq\U$, define the stacked operator
\[
T_Cx=(T_ux)_{u\in C},
\]
with Euclidean product norm. Let $p^*(K)$ be the least $p$ permitting exact recovery for all tasks using at most $K$ advice symbols, and set $p_b^*=p^*(\min\{2^b,m\})$. For $\tau\ge0$, write
\[
\rnum_\tau(L)=\#\{j:\sigma_j(L)>\tau\}.
\]

The contribution is one precise law and its consequences. The exact frontier is the minimum, over task partitions allowed by the advice alphabet, of the largest stacked rank. The approximate frontier is trapped between two singular-value partition problems, and the $\sqrt{|C|}$ separation between them is unavoidable without further task structure. A common task subspace remains as an irreducible floor; a bounded cumulative-overlap parameter quantifies how close a general family is to the weighted direct-sum law. Finding the optimal advice partition is strongly NP-hard, but in the direct-sum regime the problem becomes identical-machine makespan scheduling and therefore admits efficient approximation schemes. The empirical section does not ``validate'' these proofs; it asks whether the same geometry remains informative when the task operators come from learned representations and a scientific field.

\section{Exact and approximate frontiers}

We first isolate the classical single-operator width fact that will be used as machinery.

\begin{lemma}[Euclidean continuous width]\label{lem:width}
Let $L:\R^n\to\R^q$ be linear.  Among continuous encoders $\phi:\B_2^n\to\R^p$ and continuous decoders $\psi:\R^p\to\R^q$,
\[
 \inf_{\phi,\psi}\sup_{\norm{x}_2\le1}
 \norm{Lx-\psi(\phi(x))}_2
 =\sigma_{p+1}(L).
\]
In particular, exact recovery is possible iff $p\ge\rank(L)$.
\end{lemma}

\begin{proof}
The truncated singular-value decomposition gives the upper bound.  For the lower bound, restrict to the unit sphere in the span of the first $p+1$ right singular vectors.  Borsuk--Ulam gives $x$ with $\phi(x)=\phi(-x)$.  A common decoded value cannot lie within distance $<\norm{Lx}_2$ of both $Lx$ and $-Lx$, while $\norm{Lx}_2\ge\sigma_{p+1}(L)$.  The exact statement follows by setting the error to zero.  Equivalently, the lower bound is the linear-image specialization of the continuous nonlinear-width lower bound of DeVore--Howard--Micchelli \cite[Theorem~3.1]{DeVoreHowardMicchelli1989}; in Hilbert space the corresponding Bernstein/Kolmogorov widths are the singular values, as developed systematically by Pinkus \cite[Chapter~IV]{Pinkus1985}.
\end{proof}

\begin{theorem}[Exact task-information--state frontier]\label{thm:exact}
For every advice alphabet size $1\le K\le m$,
\begin{equation}
\boxed{
 p^*(K)
 =
 \min_{\substack{\Pcal\text{ partition of }\U\\|\Pcal|\le K}}
 \max_{C\in\Pcal}\rank(T_C).}
\label{eq:frontier}
\end{equation}
Consequently,
\[
 p_b^*=p^*(\min\{2^b,m\}).
\]
Equivalently, if
\[
 \kappa(p):=
 \min\Bigl\{|\Pcal|:\Pcal\text{ partitions }\U,
 \ \max_{C\in\Pcal}\rank(T_C)\le p\Bigr\},
\]
then the minimum number of task-advice bits required for a $p$-dimensional exact state is
\begin{equation}
\boxed{b^*(p)=\ceil{\log_2\kappa(p)}.}
\label{eq:dualfrontier}
\end{equation}
\end{theorem}

\begin{proof}
An advice map with at most $K$ values is exactly a partition $\Pcal$ of $\U$ into at most $K$ cells.  Fix one cell $C$.  If every $T_u$, $u\in C$, factors through the same $p$-dimensional state, then their product map $T_C$ also factors through that state.  By \cref{lem:width}, exact recovery therefore requires $p\ge\rank(T_C)$.  Hence every advice partition needs
$p\ge\max_{C\in\Pcal}\rank(T_C)$.

Conversely, for each cell $C$, choose a rank factorization $T_C=D_CE_C$ with $E_C:\R^n\to\R^{\rank(T_C)}$.  Store $E_Cx$ and, after $u$ is revealed, select the corresponding output block of $D_CE_Cx$.  Padding lower-dimensional cells with zeros gives a common state dimension $\max_C\rank(T_C)$.  Minimizing over partitions with at most $K$ cells proves \eqref{eq:frontier}; the $b$-bit statement follows by $K=\min(2^b,m)$, and \eqref{eq:dualfrontier} is the same statement inverted.
\end{proof}

The theorem turns future-task uncertainty into a concrete geometric object: the optimal partition of the task family under the rank of its joint row span.  Two task families with the same number of tasks and the same rank for every individual task may therefore have entirely different frontiers.

\begin{corollary}[Common-core floor and load-balancing envelope]\label{cor:core}
Let $W_u=\row(T_u)$ and
\[
 V:=\bigcap_{u\in\U}W_u,\qquad c:=\dim V,\qquad
 R:=\rank(T_\U),\qquad r_{\max}:=\max_u\rank(T_u).
\]
For each task let $\overline W_u:=W_u/V$ and $g_u:=\dim\overline W_u=\rank(T_u)-c$.  Define the weighted $K$-cell makespan
\begin{equation}
 M_K(g):=
 \min_{\substack{\Pcal\text{ partition of }\U\\|\Pcal|\le K}}
 \max_{C\in\Pcal}\sum_{u\in C}g_u.
 \label{eq:makespan}
\end{equation}
Then
\begin{equation}
\boxed{
 \max\left\{r_{\max},\ c+\ceil{\frac{R-c}{K}}\right\}
 \le p^*(K)\le c+M_K(g).}
\label{eq:coreenvelope}
\end{equation}
If the quotient innovations are globally direct,
\[
 \bigoplus_{u\in\U}\overline W_u,
\]
then the upper bound is exact:
\begin{equation}
\boxed{p^*(K)=c+M_K(g).}
\label{eq:weightedcoreexact}
\end{equation}
In particular, if every $g_u=g$, then
\begin{equation}
\boxed{p^*(K)=c+g\ceil{\frac{m}{K}}.}
\label{eq:coreexact}
\end{equation}
Thus advice can serialize task-specific innovations, but it cannot remove the common $c$-dimensional floor.
\end{corollary}

\begin{proof}
Every cell span contains $V$. Passing to the quotient by $V$, the total quotient dimension $R-c$ is at most the sum of the cell quotient dimensions, so one cell has quotient dimension at least $\ceil{(R-c)/K}$; the single-task lower bound gives $r_{\max}$. For the upper bound, for every cell $C$,
\[
 \dim\Bigl(\sum_{u\in C}\overline W_u\Bigr)
 \le \sum_{u\in C}g_u,
\]
so any weighted partition with makespan $M$ gives retained dimension at most $c+M$. Minimizing yields \eqref{eq:coreenvelope}. Under global directness the displayed inequality is equality for every cell, giving \eqref{eq:weightedcoreexact}; equal weights are minimized by a balanced partition.
\end{proof}

The direct-sum hypothesis is sufficient but not necessary for the load-balancing description to remain accurate.  The next quantity measures the total linear dependence that can accumulate inside one unresolved advice cell.

\begin{proposition}[Bounded-overlap near-tightness]\label{prop:overlap}
With the notation of \cref{cor:core}, define the cumulative overlap deficit
\begin{equation}
 \Delta:=\max_{C\subseteq\U}
 \left\{\sum_{u\in C}g_u-
 \dim\Bigl(\sum_{u\in C}\overline W_u\Bigr)\right\}.
 \label{eq:overlapdeficit}
\end{equation}
Then for every $K$,
\begin{equation}
\boxed{
 \max\{r_{\max},\ c+M_K(g)-\Delta\}
 \le p^*(K)\le c+M_K(g).}
\label{eq:overlapbracket}
\end{equation}
Hence a family whose task-specific quotient spaces have at most $\Delta$ dimensions of cumulative redundancy in any cell lies within additive $\Delta$ of the weighted direct-sum frontier.  The case $\Delta=0$ recovers \eqref{eq:weightedcoreexact}.
\end{proposition}

\begin{proof}
For every cell $C$, the definition of $\Delta$ gives
\[
 \rank(T_C)=c+\dim\Bigl(\sum_{u\in C}\overline W_u\Bigr)
 \ge c+\sum_{u\in C}g_u-\Delta.
\]
Taking the maximum over cells and then the minimum over partitions yields the lower bound; the upper bound is \cref{cor:core}.
\end{proof}

The variational characterization is explicit, but the optimal advice code need not be computationally easy to find.

\begin{corollary}[Optimal advice partition is strongly NP-hard]\label{cor:nphard}
Given an explicitly represented linear task family and an advice budget $K$, deciding whether $p^*(K)\le B$ is strongly NP-complete, even when the task row spaces are pairwise disjoint coordinate subspaces and there is no common core.
\end{corollary}

\begin{proof}
Membership in NP follows by supplying a partition and computing the ranks of its stacked operators.  For hardness, reduce from strongly NP-complete \textsc{3-Partition} \cite{GareyJohnson1979}.  Given a \textsc{3-Partition} instance with integers $a_1,\ldots,a_{3J}$ satisfying $\sum_i a_i=JB$ and $B/4<a_i<B/2$, create $3J$ tasks whose row spaces are mutually disjoint coordinate subspaces $V_i$ with $\dim V_i=a_i$, and set the advice budget to $K:=J$.  Then for every cell $C$,
\[
\rank(T_C)=\sum_{i\in C}a_i.
\]
Hence $p^*(J)\le B$ exactly when the tasks can be partitioned into $J$ triples of sum $B$.  Strong NP-hardness makes the explicit ambient dimension $\sum_i a_i$ polynomially bounded in the restricted instances, so the reduction has polynomial size.
\end{proof}

Strong NP-hardness does not mean that every structured instance is algorithmically hopeless.  In the heterogeneous direct-sum regime, the partition problem is exactly a classical scheduling problem.

\begin{proposition}[Approximation in heterogeneous direct-sum families]\label{prop:ptas}
Assume the globally direct quotient structure of \cref{cor:core}.  Then minimizing retained state is identical, after subtracting the irreducible core $c$, to scheduling jobs of sizes $g_1,\ldots,g_m$ on $K$ identical machines to minimize makespan.  Consequently, longest-processing-time-first (LPT) produces a partition satisfying, for $K\ge2$,
\begin{equation}
 p_{\rm LPT}(K)-c
 \le \left(\frac43-\frac{1}{3K}\right)\bigl(p^*(K)-c\bigr),
 \label{eq:lpt}
\end{equation}
and for every fixed $\eta>0$ there is a polynomial-time algorithm with
\begin{equation}
 p_{\eta}(K)-c\le(1+\eta)\bigl(p^*(K)-c\bigr).
 \label{eq:ptas}
\end{equation}
Thus the exact advice design can be strongly NP-hard while still admitting a polynomial-time approximation scheme on this natural structural class.
\end{proposition}

\begin{proof}
By \eqref{eq:weightedcoreexact}, subtracting $c$ leaves exactly the objective
$\min_{|\Pcal|\le K}\max_C\sum_{u\in C}g_u$, the identical-machine makespan problem $P||C_{\max}$.  The LPT guarantee is Graham's bound \cite{Graham1969}; the polynomial-time approximation scheme is due to Hochbaum and Shmoys \cite{HochbaumShmoys1987}. Adding the fixed core $c$ back gives \eqref{eq:lpt}--\eqref{eq:ptas}.
\end{proof}

Exact rank is not enough for approximate systems, so the second main result prices a nonzero tolerance.

\begin{theorem}[Robust approximate frontier]\label{thm:approx}
For a cell $C$, define the optimal branchwise error
\[
 E_p^{\max}(C):=
 \inf_{\phi_C,\{\psi_u\}}
 \sup_{\norm{x}_2\le1}\max_{u\in C}
 \norm{T_ux-\psi_u(\phi_C(x))}_2.
\]
Then
\begin{equation}
 \frac{\sigma_{p+1}(T_C)}{\sqrt{|C|}}
 \le E_p^{\max}(C)
 \le \sigma_{p+1}(T_C).
\label{eq:sandwich}
\end{equation}
Let $p^*_{K,\eps}$ be the least retained dimension achieving branchwise error at most $\eps$ with an advice alphabet of size at most $K$.  Then
\begin{equation}
\boxed{
 \min_{\substack{|\Pcal|\le K}}
 \max_{C\in\Pcal}\rnum_{\sqrt{|C|}\eps}(T_C)
 \ \le\ p^*_{K,\eps}\ \le
 \min_{\substack{|\Pcal|\le K}}
 \max_{C\in\Pcal}\rnum_{\eps}(T_C).}
\label{eq:approxfrontier}
\end{equation}
The $b$-bit quantity is $p^*_{\min(2^b,m),\eps}$, and the exact frontier is recovered at $\eps=0$.
\end{theorem}

\begin{proof}
A decoder for the stacked operator controls each branch, so \cref{lem:width} gives $E_p^{\max}(C)\le\sigma_{p+1}(T_C)$.  Conversely, branchwise error at most $e$ gives stacked Euclidean error at most $\sqrt{|C|}e$; therefore \cref{lem:width} implies $\sigma_{p+1}(T_C)\le\sqrt{|C|}E_p^{\max}(C)$.  Thresholding these inequalities and optimizing over partitions with at most $K$ cells gives \eqref{eq:approxfrontier}.
\end{proof}

\begin{proposition}[The $\sqrt{|C|}$ sandwich is sharp]\label{prop:sharp}
Fix a cell size $q\ge1$ and a retained dimension $p$.  The constants in \eqref{eq:sandwich} cannot be improved uniformly over linear task families.

If $T_1=\cdots=T_q=L$, then
\begin{equation}
 E_p^{\max}(C)=\sigma_{p+1}(L)
 =\frac{\sigma_{p+1}(T_C)}{\sqrt q},
 \label{eq:lowersharp}
\end{equation}
so the lower factor $q^{-1/2}$ is attained.  If $T_1=L$ and $T_2=\cdots=T_q=0$, then
\begin{equation}
 E_p^{\max}(C)=\sigma_{p+1}(L)
 =\sigma_{p+1}(T_C),
 \label{eq:uppersharp}
\end{equation}
so the upper factor is attained.  Whenever $\sigma_{p+1}(L)>0$, both equalities are nontrivial.  In particular, at $K=1$ each endpoint of the numerical-rank bracket in \eqref{eq:approxfrontier} is exact for a suitable task family.
\end{proposition}

\begin{proof}
For identical branches, recovering all tasks is the same problem as recovering $Lx$ once, so \cref{lem:width} gives $E_p^{\max}(C)=\sigma_{p+1}(L)$.  The stacked operator is $T_Cx=(Lx,\ldots,Lx)$, whose nonzero singular values are $\sqrt q$ times those of $L$, proving \eqref{eq:lowersharp}.  With one active branch, the branchwise problem again reduces to $L$, while zero padding leaves the nonzero singular values of the stack unchanged, proving \eqref{eq:uppersharp}.
\end{proof}

\begin{corollary}[Positive-tolerance partition design remains strongly NP-hard]\label{cor:approxhard}
Fix any constant $\eps>0$.  Given an explicitly represented linear task family and advice budget $K$, deciding whether $p^*_{K,\eps}\le B$ is strongly NP-hard.
\end{corollary}

\begin{proof}
Use the disjoint-coordinate reduction from \cref{cor:nphard}, with $3J$ tasks and advice budget $K=J$, but multiply every task operator by the integer
\[
 \lambda=(\lceil\eps\rceil+1)3J.
\]
Every nonzero singular value of every stacked operator is then $\lambda$.  If $p<\rank(T_C)$, \eqref{eq:sandwich} gives
\[
 E_p^{\max}(C)\ge \frac{\lambda}{\sqrt{|C|}}
 \ge \frac{\lambda}{\sqrt{3J}}>\eps.
\]
Thus error at most $\eps$ is possible on a cell exactly when its retained dimension is at least its rank.  The approximate instance therefore has the same feasible advice partitions as the exact \textsc{3-Partition} reduction.  The scaling is polynomial in the strongly bounded instance size.
\end{proof}

\section{Three demonstrations of the new frontier}

The examples below are consequences of \cref{thm:exact,thm:approx}; they are not substitutes for them.  Each asks a live resource question in a different field.

\subsection{Long-context attention: a well-conditioned exact memory frontier}

A rank gap is not useful if it disappears numerically.  The following softmax construction is deliberately well conditioned.

Let $s=mg$, $d_k=s$, take keys $k_j=e_j\in\R^s$, and queries
\[
 q_i=\sqrt{s}\,\alpha e_i,
 \qquad \alpha=\log(s+1).
\]
Then scaled dot-product attention has logits $\alpha I_s$ and therefore
\begin{equation}
 A=\operatorname{softmax}(\alpha I_s)
 =\frac12 I_s+\frac{1}{2s}\mathbf1\mathbf1^\top.
\label{eq:softmaxA}
\end{equation}
Its singular values are exactly $1$ once and $1/2$ with multiplicity $s-1$; hence $\kappa_2(A)=2$.  Partition the $s$ attention rows into $m$ disjoint continuation blocks $I_u$ of $g$ rows and let $R_u$ select block $I_u$.  For values with width $d_v$, the task operator is
\[
 \widetilde T_u=I_{d_v}\otimes(R_uA).
\]
Because $A$ is invertible, any stacked set of $|C|$ distinct row blocks has rank $g|C|$, so \cref{thm:exact} gives
\begin{equation}
\boxed{
 p^*(K)=d_vg\ceil{\frac{m}{K}},\qquad 1\le K\le m.}
\label{eq:attnfrontier}
\end{equation}
For $b$ advice bits, set $K=\min(2^b,m)$.
Moreover $A^2\succeq\tfrac14I$.  Every row-selected Gram matrix
$R_CA^2R_C^\top$ therefore has eigenvalues in $[1/4,1]$, so every nonzero singular value of $R_CA$ lies in $[1/2,1]$.  Thus the rank in \eqref{eq:attnfrontier} is stable at every numerical threshold $\tau<1/2$; the gap is not an exact-arithmetic Vandermonde artifact.

For $m=512$, $g=8$, and $d_v=128$:
\begin{center}
\begin{tabular}{c@{\qquad}rrrrr}
\toprule
advance task bits $b$ &0&1&3&6&9\\
\midrule
required coordinates $p_b^*$&524288&262144&65536&8192&1024\\
\bottomrule
\end{tabular}
\end{center}
Nine bits identify one of $512$ continuation blocks and reduce retained state by exactly $512\times$ while the attention operator itself remains uniformly well conditioned.  This is complementary to empirical task-aware KV methods such as DynamicKV \cite{Zhou2025}: the theorem prices an exact pre-continuation state requirement rather than proposing a token-eviction rule.

\subsection{Domain-decomposed digital twins: interface state plus regional innovations}

Consider a finite-element or reduced-order state after domain decomposition,
\[
 x=(x_\Gamma,x_1,\dots,x_m),
 \qquad x_\Gamma\in\R^{r_\Gamma},\quad x_u\in\R^{g_u}.
\]
Here $x_\Gamma$ represents globally coupled interface/trace coordinates and $x_u$ the interior reduced coordinates for subdomain $u$.  Suppose a digital twin is solved once, compressed, and only later asked for detailed state in one region; task $u$ requires
\[
 T_ux=(x_\Gamma,x_u).
\]
This is an idealized coordinate model of the interface/interior structure produced by static condensation and domain-decomposition reductions \cite{ToselliWidlund2005}: every regional query shares the interface state and has a local interior innovation.  We posit the exact direct-sum coordinates here rather than deriving them from a particular discretization.

If $g_u=g$, \cref{cor:core} gives
\begin{equation}
\boxed{
 p^*(K)=r_\Gamma+g\ceil{\frac{m}{K}},\qquad 1\le K\le m.}
\label{eq:digitaltwin}
\end{equation}
For $m=64$, $r_\Gamma=256$, and $g=128$, the exact state drops
\[
 8448\ (b=0)\quad\to\quad1280\ (b=3)\quad\to\quad384\ (b=6).
\]
The final $384=256+128$ coordinates are irreducible: once the region is known, the common interface and that region's local state still have to survive.

The heterogeneous case is more informative.  If regional reduced dimensions are $g_1,\dots,g_m$, then \cref{thm:exact} becomes
\begin{equation}
 p^*(K)=r_\Gamma+
 \min_{\substack{\Pcal\text{ partition of }[m]\\|\Pcal|\le K}}
 \max_{C\in\Pcal}\sum_{u\in C}g_u.
\label{eq:weightedrom}
\end{equation}
Thus the optimal $b$-bit region code should balance \emph{reduced state dimension}, not number of geographic regions.  This gives a direct design rule for memory-limited digital twins: metadata partitions should follow task-conditioned ROM complexity.  Standard domain-decomposition machinery supplies the interface/interior split; the new point is the exact advice--state frontier induced by uncertainty about the future regional query.

\subsection{Hierarchical multi-task representations: task similarity changes the value of bits}

Multi-task representation compression increasingly exploits dependencies among downstream tasks \cite{Huang2026,Nakanoya2021}.  The rank frontier shows exactly how such dependencies alter the value of task information.

Index tasks by $(j,\ell)$ with $j\in[q]$ a task family and $\ell\in[L]$ a task within that family.  Assume their required row spaces have the orthogonal/direct-sum structure
\begin{equation}
 W_{j,\ell}=W_0\oplus G_j\oplus V_{j,\ell},
\qquad
\dim W_0=c,
\quad\dim G_j=h,
\quad\dim V_{j,\ell}=g,
\label{eq:hierarchy}
\end{equation}
with all $G_j$ mutually independent and all $V_{j,\ell}$ mutually independent of each other and of the common/group spaces.  Here $W_0$ is universal task information, $G_j$ is family-level information, and $V_{j,\ell}$ is task-specific information.

\begin{proposition}[Exact two-level task-information frontier]\label{prop:hierarchy}
Assume $h\ge gL$.  For every advice alphabet size $q\le K\le qL$, let
\[
 k_K:=\min\left\{L,\left\lfloor\frac{K}{q}\right\rfloor\right\}.
\]
Then
\begin{equation}
\boxed{
 p^*(K)=c+h+g\ceil{\frac{L}{k_K}}.}
\label{eq:hierarchyfrontier}
\end{equation}
Thus once the advice alphabet is large enough to separate the $q$ task families, the remaining symbols should be distributed as evenly as possible among families.  No closed form is asserted here for $K<q$; those cases are still characterized exactly by \cref{thm:exact}.
\end{proposition}

\begin{proof}
Because $K\ge q$, there is a feasible no-mixing partition that assigns at least one advice cell to each family.  Even the coarsest such construction has maximum rank at most $c+h+gL\le c+2h$.  By contrast, any cell containing tasks from two distinct families contains both independent group spaces and at least two independent task innovations, so its rank is at least $c+2h+2g>c+2h$.  Hence no optimal partition mixes families.

It remains to optimize within the no-mixing class.  If family $j$ receives $k_j\ge1$ cells, balanced partitioning within that family gives maximum rank
\[
 c+h+g\ceil{L/k_j}.
\]
Since $\sum_j k_j\le K$, some family has $k_j\le\lfloor K/q\rfloor=k_K$, which gives the lower bound in \eqref{eq:hierarchyfrontier}.  Distributing the available cells as evenly as possible among the $q$ families and balancing tasks within each family attains it.
\end{proof}

Take $q=8$, $L=16$, $c=64$, $h=256$, and $g=8$.  With no advice, the joint state has rank
\[
 64+8(256)+128(8)=3136.
\]
At $b=3$ bits, the advice can identify the task family and \cref{prop:hierarchy} gives $448$ coordinates.  Additional bits then refine within the family.  The no-advice value is obtained directly from \cref{thm:exact}:
\begin{center}
\begin{tabular}{c@{\qquad}rrrrrr}
\toprule
$b$&0&3&4&5&6&7\\
\midrule
$p_b^*$&3136&448&384&352&336&328\\
\bottomrule
\end{tabular}
\end{center}
The $328=64+256+8$ floor is the exact single-task requirement.  The first three bits are therefore far more valuable than the last four because they resolve a high-dimensional shared task-family component.  The cases $b=1,2$ remain exactly covered by \cref{thm:exact} but are outside the closed-form regime $K\ge q$ of \cref{prop:hierarchy}.  The frontier prices \emph{task dependency geometry}, not task count alone.

\section{Empirical stress tests}\label{sec:experiments}

The theorems are worst-case statements on the unit state ball. The experiments test two narrower questions: whether learned or scientific task banks exhibit the predicted frontiers, and whether task-aware grouping can reduce the retained dimension compared with uninformed grouping. Test data are never used to choose $\eps$, a partition, or $p$.

\subsection{Protocol}

\paragraph{Learned vision operators.}
We use CIFAR-10 \cite{Krizhevsky2009}, Fashion-MNIST \cite{Xiao2017}, and SVHN \cite{Netzer2011}. A fixed ImageNet-pretrained ResNet-50 \cite{He2016} maps each image to a $2048$-dimensional state. For each of three seeds ($11,29,47$), a multiclass linear head is fit on the training representation. Its logits are centered across classes, giving ten scalar one-vs-rest margin operators. Each nonzero task row is divided by its Euclidean norm, so $\eps$ has a uniform operator interpretation: $\eps=0.1$ permits at most ten percent of one task's maximal linear-score magnitude on the unit ball. Fifteen percent of the training set is held out for calibration; the official test set remains untouched until evaluation. We use $K\in\{1,2,4,5,10\}$.

\paragraph{Scientific operators.}
The Burgers experiment uses the standard $256\times100$ shock solution array commonly used with the viscous Burgers benchmark \cite{Raissi2019}. Each time snapshot is scaled into the unit Euclidean ball. Twelve overlapping local spatial averages, equally centered over the domain with window width $0.18$ of the spatial span, define the task bank. The first $70$ snapshots are calibration states and the last $30$ are held out in time. We use $K\in\{1,2,3,4,6,12\}$.

For a fixed partition $\Pcal$, let $e_C(p)$ denote the worst scalar-branch operator norm after retaining the first $p$ right singular directions of $T_C$. We report the \emph{constructive certified dimension}
\[
 p_{\mathrm{cert}}(\Pcal,\eps)=\max_{C\in\Pcal}\min\{p:e_C(p)\le\eps\}.
\]
It is a feasible dimension for that explicit truncated-SVD construction; it need not equal the globally optimal branchwise dimension. The ``upper-bound DP'' solves the finite partition problem with cell cost $\rnum_\eps(T_C)$ exactly by bitmask dynamic programming. We also evaluate a geometry-greedy partition, a spectral partition, and $100$ balanced random partitions per seed. Finite-sample RMSE thresholds are recorded separately in the supplement because they answer a distributional question, not the worst-case theorem.

\subsection{Exact-rank audit}

At $\eps=0$, the lower and upper bounds collapse to rank.  The three vision banks all have joint rank $9$ because centering ten class logits imposes one exact linear dependence.  Their shared sequence
\[
 (K,p^*(K))=(1,9),(2,5),(4,3),(5,2),(10,1)
\]
is therefore an implementation consistency check, not three independent empirical discoveries.  The Burgers bank is structurally different: its twelve local-average operators are numerically independent and give
\[
 (K,p^*(K))=(1,12),(2,6),(3,4),(4,3),(6,2),(12,1).
\]
The complete exact-frontier plot is moved to the supplement.  The empirical content of the learned task banks begins when $\eps>0$, where their singular spectra and downstream margins differ.

\subsection{Approximation makes task grouping matter}

The Burgers bank gives the clearest nontrivial approximate regime. At $\eps=0.35$ and $K=2$, the theorem interval for the upper-bound-optimal partition is
\[
4\le p^*_{2,0.35}\le5.
\]
The explicit truncated-SVD construction reaches $p_{\mathrm{cert}}=4$, attaining the theorem lower bound. Across $300$ balanced random partitions (100 for each seed), the median certified dimension is $6$ and the 2.5--97.5\% range is $5$--$6$. At $p=4$, the held-out worst standardized RMSE is $0.0368$ and the relative $\ell_2$ error across the physical task outputs is $0.0718$. Thus the saving is not obtained by accepting large test error: the structured partition stores four coordinates where a typical random grouping requires six.

\begin{figure}[t]
\centering
\includegraphics[width=.78\linewidth]{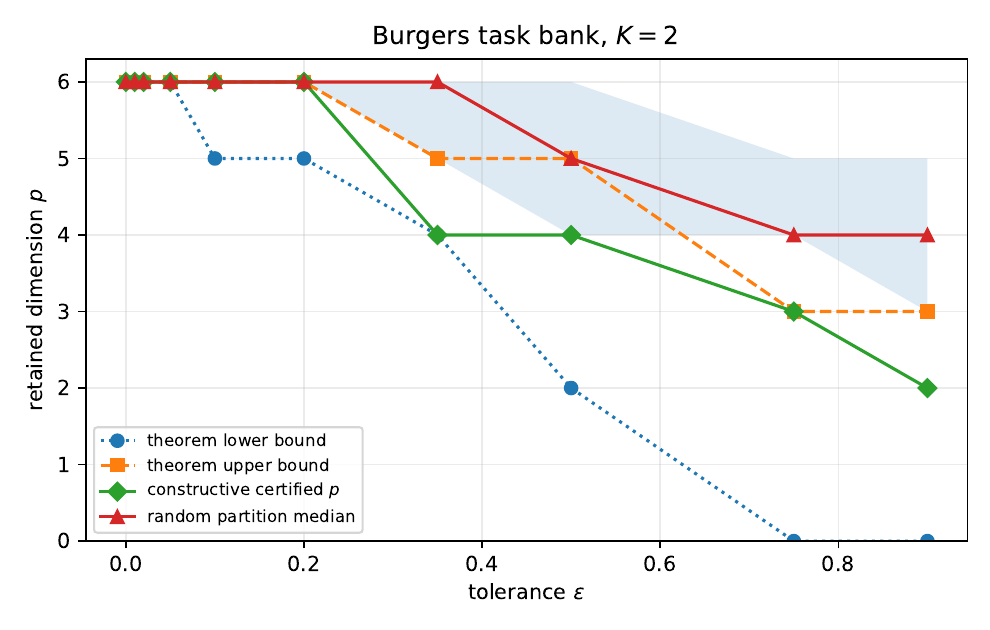}
\caption{Approximate Burgers frontier for $K=2$. The plot separates the theorem interval from a constructive certificate and from the distribution of uninformed task groupings. At $\eps=0.35$, the certificate reaches $p=4$ while the random-partition median is $6$.}
\label{fig:burgfrontier}
\end{figure}

The same point is visible directly in the error curve. At $K=2$, increasing $p$ from $3$ to $4$ crosses the uniform operator tolerance: the branch error falls from $0.5295$ to $0.3206$. The held-out worst RMSE at $p=4$ is $0.0368$, well below the declared operator tolerance; see \cref{fig:burgerror}. The theorem controls all states in the unit ball, whereas the test trajectory occupies only a small subset of that ball.

\begin{figure}[t]
\centering
\includegraphics[width=.76\linewidth]{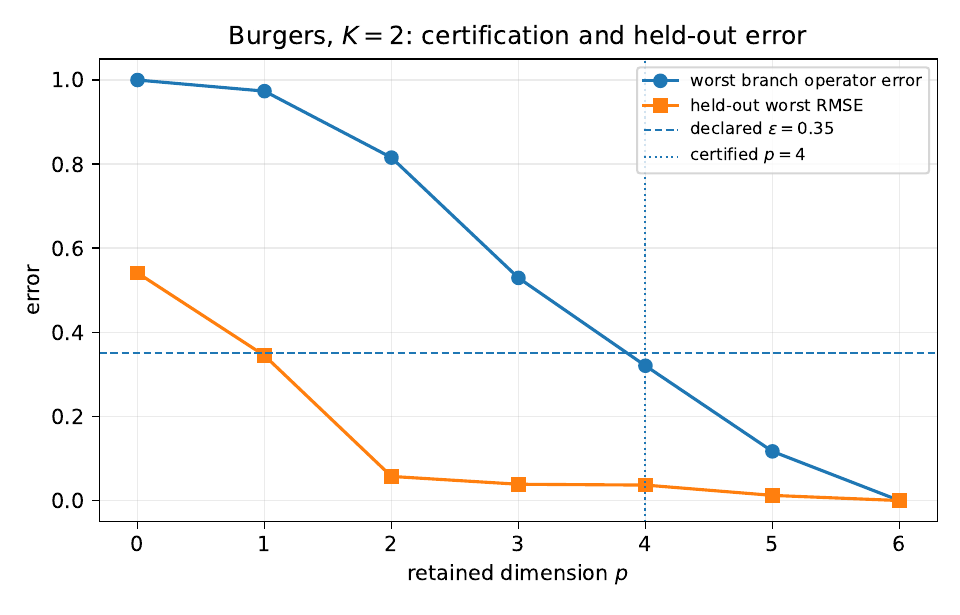}
\caption{Burgers, $K=2$, $\eps=0.35$. The vertical line marks the constructive certified dimension. Uniform branch error and held-out trajectory error answer different questions and are therefore reported separately.}
\label{fig:burgerror}
\end{figure}

\subsection{Learned task banks: certified score preservation is not label preservation}

The vision operators enter a nontrivial approximate regime at larger tolerance. For $K=2$ and $\eps=0.5$, the exact retained dimension is $5$ on all three datasets, while the constructive certificate uses $p=4$. The theorem intervals are $[0,5]$ for CIFAR-10 and $[1,5]$ for Fashion-MNIST and SVHN. At that certified dimension, held-out worst standardized score RMSE is $0.0107$, $0.0333$, and $0.00376$, respectively. Top-1 accuracy drops by $1.69$, $11.0$, and $2.55$ percentage points.

\begin{table}[t]
\centering
\caption{Learned vision task banks at $K=2$, $\eps=0.5$. ``Top-1 drop'' is in percentage points and is evaluated at the constructive certified dimension.}
\label{tab:visionapprox}
\small
\begin{tabular}{@{}lcccc@{}}
\toprule
dataset & theorem interval & $p_{\rm cert}$ & worst test RMSE & top-1 drop \\
\midrule
CIFAR-10 & $[0,5]$ & 4 & 0.0107 & 1.69 \\
Fashion-MNIST & $[1,5]$ & 4 & 0.0333 & 11.00 \\
SVHN & $[1,5]$ & 4 & 0.00376 & 2.55 \\
\bottomrule
\end{tabular}
\end{table}

This variation is informative. The theorem certifies the linear task scores; an argmax classifier can still be sensitive when class margins are small. Conversely, a finite data distribution can look much easier than the unit-ball problem. In the raw runs, a validation-RMSE selector sometimes chose $p=0$ even while classification accuracy collapsed. We therefore do not use that quantity as evidence for the theorem. It is reported in the supplement as a separate distributional notion of sufficiency. The experiment makes the distinction concrete: \emph{uniform task preservation, average score fidelity, and downstream decision stability are not interchangeable}.

\subsection{What the experiments establish}

The empirical evidence supports three limited claims. First, the implementation recovers the exact rank law on both centered-logit and PDE task banks; for vision this is deliberately treated as an algebraic audit because the shared rank-$9$ structure is imposed by logit centering. Second, the approximate interval is operationally nontrivial: the Burgers construction can hit the lower bound, and informed grouping can require fewer coordinates than random grouping at the same uniform tolerance. Third, worst-case and distributional sufficiency separate in practice. None of these observations enlarges the theorem; they show where its resource quantity is useful and where another quantity must be introduced.

\section{Relation to existing theory and novelty boundary}

Several neighboring theories supply essential ingredients.  Predictive-state representations characterize state by predictions of future tests and relate linear state dimension to rank \cite{LittmanSuttonSingh2001}.  Functional and multi-functional source coding study how much communication is needed to compute one or several functions, often with decoder side information and stochastic sources \cite{FeiziMedard2009}.  Task-aware co-design shows, in a linear single-task setting, that a representation of dimension $\rank(K)$ can suffice for a task operator $K$ \cite{Nakanoya2021}.  Linear sketching and workload-aware query mechanisms construct one compressed representation or strategy that supports a declared workload \cite{Woodruff2014,LiHayRastogiMiklauMcGregor2010}.  Nonlinear-width theory supplies the single-operator identity used in \cref{lem:width} \cite{DeVoreHowardMicchelli1989,Pinkus1985}.  Recent systems work adapts compression to the downstream task in KV memory and multi-task visual representation coding \cite{Zhou2025,Huang2026}.

Our contribution is narrower than all of these and different in ordering.  The source $x$ is fully available; only a bounded message about the \emph{future task} is known when the state is formed; and the exact task is disclosed after compression.  The state is priced by continuous coordinate dimension under worst-case exact or uniform approximate recovery.  Unlike a fixed known workload, the workload itself is only partially resolved at encoding time.  In this model, the per-cell width identity is classical, while \cref{thm:exact} identifies the joint task-information--state frontier as an optimal partition by stacked rank.  \Cref{cor:core,prop:overlap} quantify the effect of shared and overlapping task geometry, \cref{thm:approx,prop:sharp} give a robust singular-value frontier with globally sharp constants, and \cref{cor:nphard,prop:ptas,cor:approxhard} separate exact computational hardness from approximability on a natural direct-sum class.  We are not aware of an earlier result giving this advice-cardinality/stacked-rank frontier under the stated information order.  The novelty claim is restricted to this model and does not claim priority over task-aware compression, functional coding, predictive state, workload sketching, or classical scheduling approximation.

\section{Scope and conclusion}

The theorem prices continuous state coordinates, not finite-precision bits. The task family is finite and linear. Nonlinear tasks require a different complexity object---for example a local Jacobian family, a nonlinear width, or a task-defined quotient geometry. Optimal advice design is strongly NP-hard in general, but \cref{prop:ptas} shows why hardness alone is not the right practical conclusion: a heterogeneous direct-sum family has the same weighted frontier and admits a PTAS. What remains open is a comparable approximation theory for genuinely overlapping task geometries.

These limits leave the central result intact. When the state is formed under partial knowledge of its future use, the information that must survive is exactly the joint task-visible subspace inside each unresolved advice cell. More task information is useful only when it separates tasks whose joint row span is expensive. A common task core cannot be removed; bounded overlap places the frontier near a weighted load-balancing law; task-specific innovations can be serialized as uncertainty about the future task is resolved. In the approximate regime, singular spectra replace exact rank, and \cref{prop:sharp} shows that the $\sqrt{|C|}$ gap cannot be removed without structural assumptions.

The next mathematical targets are therefore sharper structure-dependent approximate frontiers for correlated task families, approximation guarantees beyond the direct-sum regime, and nonlinear task families with a certifiable analogue of joint task-visible dimension. The present paper fixes the linear benchmark against which those extensions can be measured.

\section*{Data and code availability}
The submission package contains the experiment code, fixed configuration, post-processing script, result snapshots used for every reported empirical number, and compilation scripts. The code accepts a data root at run time and does not require hard-coded local paths. The reported vision experiments use CIFAR-10, Fashion-MNIST and SVHN; the scientific experiment uses a standard Burgers shock array. The supplied result manifests record software versions, seeds, split rules, and the fact that held-out test data were not used to select tolerance, partition, or retained dimension.

\end{document}